\documentclass[4p]{elsarticle}
\usepackage{ulem}
\usepackage{times}
\usepackage{url}
\usepackage{latexsym}
\usepackage{amsmath}
\usepackage{amsthm}

\usepackage{booktabs}

\usepackage[english]{babel}

\usepackage[T1]{fontenc}
\usepackage{amsfonts}
\usepackage{amssymb}
\usepackage{caption}
\usepackage{float}
\usepackage{url}
\usepackage{tikz}
 \usepackage[numbers]{natbib}
\usetikzlibrary{automata}
\usetikzlibrary{arrows}
\usetikzlibrary{shapes}
\usetikzlibrary{decorations.pathmorphing}
\usepackage[lined,algonl,boxed]{algorithm2e}
\usepackage{algorithmic}
\usepackage{hyperref}

\tikzstyle{every picture}=[
  >=stealth', shorten >= 1pt, node distance=1.44cm,auto, bend angle=25,initial text=,
  every state/.style={inner sep = 0.75mm, minimum size=1mm},font=\scriptsize]

\newtheorem{example}{Example}
\newtheorem{definition}{Definition}
\newtheorem{proposition}{Proposition}
\newtheorem{theorem}{Theorem}
\newtheorem{corollary}{Corollary}

\begin{document}
\begin{frontmatter}
  \title{New Linear-time Algorithm for SubTree Kernel Computation based on Root-Weighted Tree Automata}

  \author[mainaddress]{Ludovic Mignot}

  \author[myaddress]{Faissal Ouardi}
  \ead{Correponding author: f.ouardi@um5r.ac.ma}

  \author[mainaddress]{Djelloul Ziadi}

  \address[mainaddress]{Groupe de Recherche Rouennais en Informatique Fondamentale, Université de Rouen Normandie, 76801 Saint-Étienne du Rouvray Cedex, France}
  \address[myaddress]{ANISSE Research Group, Department of Computer Science, Faculty of Sciences, Mohammed V University in Rabat, Morocco}

  \begin{abstract}
    Tree kernels have been proposed to be used in many areas  as the automatic learning of natural language applications.
    In this paper, we propose a new linear time algorithm based on the concept of weighted tree automata for SubTree kernel computation.
    First, we introduce a new class of weighted tree automata, called Root-Weighted Tree Automata, and their associated formal tree series. Then we define, from this class, the SubTree automata that represent compact computational models for finite tree languages.
    This allows us to design  a theoretically guaranteed linear-time  algorithm  for computing the SubTree  Kernel based on weighted tree automata intersection. The key idea behind the proposed  algorithm is to replace DAG reduction and  nodes sorting  steps used in previous approaches by  states
    equivalence classes computation allowed in the weighted tree automata approach.  Our approach has three major advantages: it is output-sensitive, it is free sensitive from the tree types (ordered trees versus unordered trees), and it is well adapted  to any incremental tree kernel based
    learning methods.
    Finally, we conduct a variety of comparative experiments on a wide range of synthetic tree languages datasets adapted for a deep algorithm analysis. The obtained results show that the proposed
    algorithm outperforms state-of-the-art methods.
  \end{abstract}

\end{frontmatter}

\section{Introduction}\label{se:int}
Kernel
methods have been widely used to extend the applicability of many well-known algorithms,
such as the Perceptron~\cite{aizerman65}, Support Vector Machines~\cite{cortes95}, or Principal Component Analysis~\cite{zelenkoAR03}.
Tree kernels are interesting approaches in areas of machine learning based natural language processing.
They have been applied to reduce such effort for several natural language tasks, e.g., relation extraction~\cite{icann97a},
syntactic parsing re-ranking~\cite{CollinsDuffy},
named entity recognition\cite{culottaS04,cumbyR03}, Semantic Role Labeling~\cite{moschitti04}, paraphrase detection~\cite{filice-etal-2015-structural} and computational argumentation~\cite{wachsmuth-etal-2017-computational}.

In~\cite{Haussler}, Haussler introduces a framework based on convolution kernel to measure the similarity between structured objects in terms of the similarities of their subparts.
Based on this idea, many convolution kernels for trees are introduced and have been successfully applied to a variety of problems.

The first proposed kernels in the context of tree structured data were the subtree (ST) kernel~\cite{Vishwanathan} and the subset tree (SST) kernel~\cite{CollinsDuffy}. The first defines a feature space consisting of the set of all proper subtrees, while the second extends this set by also considering subset trees.
The research on kernel for trees has evolved either by finding more expressive kernel functions or faster kernels. Expressivity and sparsity has been dealt  by introducing ST-like substructures as features:  the partial tree (PT) kernel~\cite{moschitti2006} and the elastic tree kernel~\cite{Kashima2002KernelsFS}, SubPath~\cite{Kimura}. For more details about kernels for trees, we recommend the thesis of Da San Martino 2009~\cite{dasanmartino}.\\
In this article, we focus on the computation of the ST kernel. The main idea of the ST kernel as introduced in~\cite{Vishwanathan} is to compute the number of common subtrees between two trees \( t_1 \) and \( t_2 \) having respectively \( m \) and \( n \) nodes. It can be recursively computed as follows:
\begin{equation}\label{eq1}
  K(t_1, t_2) = \sum\limits_{ (n_1,n_2)\in N_{t_1}\times N_{t_2}} \Delta(n_1, n_2)
\end{equation}
where \( N_{t_1} \) and \( N_{t_2} \) are the set of nodes respectively in \( t_1 \) and \( t_2 \),
\( \Delta(n_1, n_2) = \sum_{i=1}^{|{\cal S}|}I_i(n_1)\cdot I_i(n_2) \) for some finite set of subtrees \( S = \{s_1, s_2, \cdots \} \), and \( I_i(n) \) is an indicator function which is equal to \( 1 \)
if the subtree  is rooted at node \( n \) and to \( 0 \) otherwise.  Then, a string matching algorithm is used where trees are transformed into strings (see~\cite{Vishwanathan}  for details).
This algorithm has an overall computational complexity equals to \( O(\max(m,n)\log(\max(m,n))) \) which is  the best worst-case time complexity for this problem.

In~\cite{moschitti06}, Moschitti defined an algorithm for the computation of this type of tree kernels which computes the kernels
between two syntactic parse trees in \( O(m \times n) \) time, where \( m \) and \( n \) are the number of nodes in the two trees.
Thus, Moschitti modified the equation~(\ref{eq1}) by introducing a parameter
\( \sigma\in \{0,1\} \) which enables the SubTrees (\( \sigma=1 \)) or the SubSet Trees (\( \sigma=0 \)) evaluation and which is defined for
two trees \( t_1 \) and \( t_2 \) as follows:
if the productions at \( n_{1} \) and \( n_{2} \) are different, then \( \Delta (n_{1},n_{2}) = 0 \); if they are the same and \( n_{1} \) and \( n_{2} \) are
leaves, then \( \Delta (n_{1},n_{2}) = 1 \); finally if the productions at \( n_{1} \) and  \( n_{2} \) are the same, and if \( n_{1} \) and \( n_{2} \) are
not leaves then \( \Delta (n_{1},n_{2}) = \prod_{j=1}^{nc(n_{1})}(\sigma+  \Delta (C_{n_{1}}^{j}, C_{n_{2}}^{j})) \),
where \( nc(n_1) \) is the number of children of \( n_1 \) and \( C_{n}^{j} \) is \( j^{th} \) child of the node \( n \).
This algorithm can be tuned to avoid any evaluation when \( \Delta(n_1, n_2) = 0 \) by efficiently building a  node pair set \( N_p =\{(n_1 , n_2)\in N_{t_1} \times N_{t_2} : p(n_1 ) = p(n_2 )\} \), where \( p(n) \) returns
the production rule associated with \( n \). It requires the sorting of trees productions at a  pre-processing step and then compute \( \Delta \) by a dynamic programming approach. This method has a worst-case time complexity in \( O(m \times n) \) but in practical applications it provides a quite relevant speed-up.

In 2020, Azais and Ingel~\cite{Azais20} develop a unified framework based on  Direct Acyclic Graphs (DAG) reduction for computing the ST kernel from ordered or unordered trees, with or without labels on their vertices. DAG reduction of a tree forest is introduced as compact representation of common subtrees structures that makes possible fast computations of the  subtree kernel. The main advantage of this approach compared to those based on string representations used by Vishwanathan and Smola, in 2002 is that it makes possible fast repeated computations of the ST kernel. Their method allows the implementation of any weighting function, while the recursive computation of the ST kernel proposed by Aiolli {\it et al.}~\cite{4221313} also uses DAG reduction of tree data but makes an extensive use of the exponential form of the weight. They investigate the theoretical complexities of the different steps of the DAG computation 
and prove that it is in
\(O(\max(m,n))\) for ordered trees and in \(O(\max(m,n)\log(\max(m,n))\) for unordered trees (Proposition~7, \cite{Azais20}).

In the following, we propose a new method to compute the ST kernel using weighted tree automata. We begin by defining a new class of weighted tree automata that we call Rooted Weighted Tree Automata (RWTA). This class of weighted tree automata represents a new efficient and optimal  alternative for representing tree forest instead of annotated DAG representation. Then we prove that the ST kernel can be computed efficiently in linear time using a general intersection of RWTA that can be turned into a determinization of a WTA by states accessibility.

The paper is organized as follows: Section~\ref{sec prel} outlines finite tree automata
over ranked and unranked trees and regular tree languages. Next, in Section~\ref{sec rwta}, we define a new
class of weighted tree automata that we call Rooted Weighted Tree Automata.
Thus, in Section~\ref{sec tree seri}, the definitions of SubTree series and automata are obtained. Afterwards, in Section~\ref{sec algo},
we present our algorithms. The first one constructs the RWTA \( A_L \)  as a compact representation of a finite tree language \( L \) in linear time.
The second one computes the RWTA \( A_X\odot A_Y \) representing the Hadamard product of the RWTAs \( A_X \) and \( A_Y \) in time \( O(|SubTree(X)|+|SubTree(Y)|) \). Section~\ref{sec experiments} shows
the efficiency of our method by conducting extensive comparative experiments on a variety of tree languages classes.
Section~\ref{sec conclusion} concludes the paper.

\section{Root-Weighted Tree Automata}\label{sec rwta}

Let \( \Sigma \) be an alphabet.
A \emph{tree} \( t \) over \( \Sigma \) is inductively defined by \( t=f(t_1,\ldots,t_k) \) where \( k \) is any integer, \( f \) is any symbol in \( \Sigma_k \) and \( t_1,\ldots,t_k \) are any \( k \) trees over \( \Sigma \).
We denote by \( T_{\Sigma} \) the set of trees over \( \Sigma \).
A \emph{tree language} over \( \Sigma \) is a subset of \( T_\Sigma \).
We denote by \( |t| \) the size of a tree \( t \), \emph{i.e.}, the number of its nodes.
For any tree language \( L \), we set \( |L|=\sum_{t\in L}|t| \).\\
A \emph{formal tree series} \( \mathbb{P} \) over a set \( S \) is a mapping from \( T_\Sigma \) to \( S \).
Let \( \mathbb{M}=(M,+) \) be a monoid which identity is \( 0 \).
The \emph{support} of \( \mathbb{P} \) is the set \( \mathrm{Support}(\mathbb{P})=\{t\in T_\Sigma\mid \mathbb{P}(t)\neq 0\} \).
Any formal tree series is equivalent to a formal sum 
\( \mathbb{P}=\sum_{t\in T_\Sigma} \mathbb{P}(t)t \).
In this case, the formal sum is considered associative and commutative.

\begin{definition}
  Let \( \mathbb{M}=(M,+) \) be a commutative monoid.
  An \( \mathbb{M} \)-\emph{Root Weighted Tree Automaton} (\( \mathbb{M} \)-RWTA) is a 4-tuple \( (\Sigma,Q,\nu,\delta) \) with:
  \begin{itemize}
    \item \( \Sigma=\bigcup_{k\in\mathbb{N}} \Sigma_k \) an  alphabet,
    \item \( Q \) a finite set of \emph{states},
    \item \( \nu \) a function from \( Q \) to \( M \) called the \emph{root weight function},
    \item \( \delta \) a subset of \( Q\times\Sigma_k\times Q^k \), called the \emph{transition set}.
  \end{itemize}
\end{definition}
\noindent When there is no ambiguity, an \( \mathbb{M} \)-RWTA is called a RWTA\@.

The root weight function \( \nu \) is extended to \( 2^Q \rightarrow M \) for any subset \( S \) of \( Q \) by \( \nu(S)=\sum_{s\in S} \nu(s) \).
The function \( \nu \) is equivalent to the finite subset of \( Q\times M \) defined for any couple \( (q,m) \) in \( Q\times M \) by
\( (q,m)\in\nu \) \( \Leftrightarrow \) \( \nu(q)=m \).

The transition set \( \delta \) is equivalent to the function in
\( \Sigma_k\times Q^k \rightarrow 2^Q \) defined for any symbol \( f \) in \( \Sigma_k \) and
for any \( k \)-tuple \( (q_1,\ldots,q_k) \) in \( Q^k \) by
\begin{equation*}
  q\in\delta(f, q_1, \ldots, q_k) \Leftrightarrow (q, f, q_1, \ldots, q_k)\in\delta.
\end{equation*}

The function \( \delta \) is extended to \( \Sigma_k \times {(2^Q)}^k  \rightarrow 2^Q \)  as follows: for any symbol \( f \) in \( \Sigma_k \), for any
\( k \)-tuple \( (Q_1,\ldots,Q_k) \) of subsets of \( Q \),
\begin{equation*}
  \delta(f, Q_1, \ldots, Q_k)=\bigcup_{(q_1, \ldots, q_k)\in Q_1\times\cdots\times Q_k}
  \delta(f, q_1, \ldots, q_k).
\end{equation*}

Finally, the function \( \Delta \) maps a tree \( t=f(t_1,\ldots,t_k) \)
to a set of states as follows:
\begin{equation*}
  \Delta(t)=\delta(f, \Delta(t_1), \ldots, \Delta(t_k)).
\end{equation*}
When \( t\in \Sigma_0 \) then  \( \Delta(t)=\delta(t)\).

\noindent A \emph{weight} of a tree associated with \( A \) is \( \nu(\Delta(t)) \).
The \emph{formal tree series realized} by \( A \) is the formal tree series over \( M \) denoted by \( \mathbb{P}_A \) and defined by \( \mathbb{P}_A(t)=\nu(\Delta(t)) \), with \( \nu(\emptyset)=0 \) with \( 0 \) the identity of \( \mathbb{M} \).
The \emph{down language} of a state \( q \) in \( Q \) is the set \( L_q=\{t \mid  q\in \Delta(t)\} \).


\begin{example}\label{ex rwta}
  Let us consider the  alphabet \( \Sigma \) defined by \( \Sigma_0=\{a\} \), \( \Sigma_1=\{h\} \) and \( \Sigma_2=\{f\} \).
  Let \( \mathbb{M}=(\mathbb{N},+) \).
  The RWTA \( A=(\Sigma,Q,\nu,\delta) \) defined by
  \begin{align*}
    Q      & = \{1,2,3,4,5\},                      \\
    \nu    & = \{(1,0),(2,3),(3,1),(4,2),(5,4)\},  \\
    \delta & = \{(1,a),(3,a), (2,f,1,3),(4,f,3,3), \\
           & \qquad (5,h,2),(5,h,4),(5,h,5)\},
  \end{align*}
  is represented in Figure~\ref{fig ex RWTA} and realizes the tree series:
  \begin{multline*}
    \mathbb{P}_A=a+5f(a,a)+4h(f(a,a))+4h(h(f(a,a)))+\cdots \\
    +4h(h(\ldots h(f(a,a))\ldots))+\cdots
  \end{multline*}
\end{example}
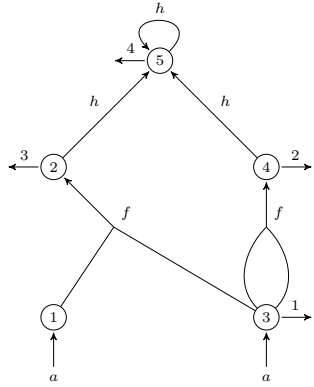
\begin{figure}[ht]
  \centering
  \captionsetup{justification=centering}
  \scalebox{0.8}{
  \begin{tikzpicture}[node distance=2.5cm,bend angle=30,transform shape,scale=1]
    \node[state] (q5) {\( 5 \)};
    \node[state, below left of=q5] (q2) {\( 2 \)};
    \node[state, below right of=q5] (q4) {\( 4 \)};
    \node[state, below of=q2] (q1) {\( 1 \)};
    \node[state, below of=q4] (q3) {\( 3 \)};
    \draw (q5) ++(-0.75cm,0cm) edge[above,<-] node {\( 4 \)} (q5);
    \draw (q2) ++(-0.75cm,0cm) edge[above,<-] node {\( 3 \)} (q2);
    \draw (q4) ++(0.75cm,0cm) edge[above,<-] node {\( 2 \)} (q4);
    \draw (q3) ++(0.75cm,0cm) edge[above,<-] node {\( 1 \)} (q3);
    \draw (q1) ++(0cm,-1cm) node {\( a \)}  edge[->] (q1);
    \draw (q3) ++(0cm,-1cm) node {\( a \)}  edge[->] (q3);
    \path[->]
    (q2) edge[->,above left] node {\( h \)} (q5)
    (q4) edge[->,above right] node {\( h \)} (q5)
    (q5) edge[loop,->,above] node {\( h \)} ()
    ;
    \draw (q1) ++(1cm,1.5cm)  edge[->] node[above right,pos=0] {\( f \)} (q2) edge[shorten >=0pt,] (q1) edge[shorten >=0pt,] (q3);
    \draw (q3) ++(0cm,1.5cm)  edge[->] node[above right,pos=0] {\( f \)} (q4) edge[shorten >=0pt,in=135,out=-135] (q3) edge[shorten >=0pt,in=45,out=-45] (q3);
  \end{tikzpicture}
  }
  \caption{The RWTA \( A \).}%
  \label{fig ex RWTA}
\end{figure}

The class of formal tree series which are realizable by the RWTAs over a monoid \( (M, +) \)
is exactly the class of recognizable step function over any semiring \((M, +, \times)\)~\cite{DrosteG07,DrosteV06}.

\section{RWTA and Tree Series}\label{sec tree seri}
An 
RWTA can be seen as a prefix tree defined in the case of words.
It is a compact structure
which allows us to represent a finite set of trees. Notice that the underlying graph of 
an acyclic
RWTA, called minimal Direct Acyclic Graphs (DAG), has been
  introduced in DAG-based algorithms~\cite{AiolliMSM06,dasanmartino}
  as a compact representation to compute efficiently different tree kernels. In the following, we introduce the Subtree series as well as their Subtree automata.

%

\subsection{Subtree Series and Subtree Automaton}

Let \( \Sigma \) be an alphabet and \( t=f(t_1,\ldots,t_k) \) be a tree in \( T_{\Sigma} \).
The set \( \mathrm{SubTree}(t) \) is the set inductively defined by\\
 $ \mathrm{SubTree}(t)= \{t\}\cup \bigcup_{1\leq j\leq k} \mathrm{SubTree}(t_j)$.\\
For example if \( t=f(h(a),f(h(a),b)) \), then\\
 $ \mathrm{SubTree}(t)=\{a, b, h(a), f(h(a), b), f(h(a), f(h(a), b))\}$.\\

Let \( L \) be a tree language over \( \Sigma \).
The set \( \mathrm{SubTreeSet}(L) \) is the set defined by \( \mathrm{SubTreeSet}(L)=\bigcup_{t\in L}\mathrm{SubTree}(t) \).

The formal tree series \( \mathrm{SubTreeSeries}_t \) is the tree series over \( \mathbb{N} \) inductively defined by \\
 $ \mathrm{SubTreeSeries}_t=t+\sum_{1\leq j \leq k} \mathrm{SubTreeSeries}_{t_j}$.

\begin{example}\label{exp tree}
  Let \( t=f(a,g(a)) \) be a tree.
  The set \( \mathrm{SubTree}_t \) of the tree \( t \) 
  is the set \( \{a, g(a), f(a, g(a)) \} \).
\end{example}


If \( L \) is finite, the rational series \( \mathrm{SubTreeSeries}_L \) is the tree series over \( \mathbb{N} \)
defined by
\begin{equation*}
  \mathrm{SubTreeSeries}_L=\sum_{t\in L} \mathrm{SubTreeSeries}_{t}.
\end{equation*}


\begin{definition}
  Let \( \Sigma \) be an alphabet.
  Let \( L \) be a finite tree language over \( \Sigma \).
  The \emph{SubTree automaton} associated with \( L \) is the RWTA \( \mathrm{STAut}_L=(\Sigma,Q,Q_m,\nu,\delta) \) where:
  \begin{itemize}
    \item \( Q=\mathrm{SubTreeSet}(L) \),
    \item \( Q_m=L \),
    \item \( \forall t\in Q \), \( \nu(t)=\mathrm{SubTreeSeries}_L(t) \),
    \item \( \forall t=f(t_1,\ldots,t_k)\in Q \), \( \delta(f,t_1,\ldots,t_k)=\{t\} \).
  \end{itemize}
\end{definition}

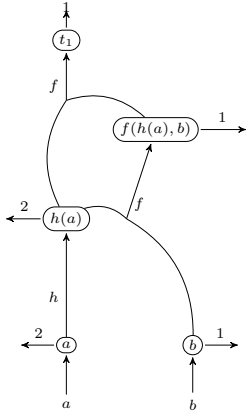
\begin{figure}[H]
\scalebox{0.8}{
  \centerline{
    \begin{tikzpicture}[node distance=2.1cm,bend angle=30,transform shape,scale=1]
      \node[state,rounded rectangle] (a) {\( a \)};
      \node[state, above of=a,rounded rectangle] (h) {\( h(a) \)};
      \node[state, right of=a,rounded rectangle] (b) {\( b \)};
      \node[state, above right of=h,rounded rectangle] (f) {\( f(h(a),b) \)};
      \node[state, above left of=f,rounded rectangle] (t1) {\( t_1 \)};
      \draw (a) ++(-0.75cm,0cm) edge[<-,above] node {\( 2 \)} (a);
      \draw (h) ++(-1cm,0cm) edge[<-,above] node {\( 2 \)} (h);
      \draw (b) ++(0.75cm,0cm) edge[<-,above] node {\( 1 \)} (b);
      \draw (t1) ++(0cm,0.5cm) edge[<-,above] node {\( 1 \)} (t1);
      \draw (f) ++(1.5cm,0cm) edge[<-,above] node {\( 1 \)} (f);
      \draw (a) ++(0cm,-1cm) node {\( a \)}  edge[->] (a);
      \draw (b) ++(0cm,-1cm) node {\( b \)}  edge[->] (b);
      \path[->]
      (a) edge[->,below left] node {\( h \)} (h);
      \draw (h) ++(1cm,0cm)  edge[->] node[above right,pos=0] {\( f \)} (f) edge[shorten >=0pt,above,bend right] (h) edge[shorten >=0pt,above,bend left] (b);
      \draw (t1) ++(0cm,-1cm)  edge[->] node[above left,pos=0] {\( f \)} (t1) edge[shorten >=0pt,above,bend right] (h) edge[shorten >=0pt,above,bend left] (f);
    \end{tikzpicture}
  }
  }
  \caption{The RWTA associated with the tree  \(f(h(a),f(h(a),b)) \).}%
  \label{fig ex At}
\end{figure}

\begin{proposition}\label{propst}
  Let \( \Sigma \) be an alphabet.
  Let \( L \) be a finite tree language over \( \Sigma \).
  Then,

  \begin{equation*}
    \mathbb{P}_{\mathrm{STAut}_L}= \mathrm{SubTreeSeries}_L.
  \end{equation*}
\end{proposition}
\begin{proof}
  Let \( t=f(t_1,\ldots,t_n) \), \( \mathrm{STAut}_t=(\Sigma,Q,Q_m,\nu,\delta) \) and \( \mathrm{STAut}_{t_i}=(\Sigma,Q_i,Q_{m_i},\nu_i,\delta_i) \) for \( 1\leq i\leq k \).
  Notice that by definition:
  \( Q=\{t\}\cup \bigcup_{1\leq i\leq k}Q_i \) and \( \delta=\{(t,f,t_1,\ldots,t_k)\}\cup \bigcup_{1\leq i\leq k}\delta_i \).

  \noindent Consequently, \( \mathbb{P}_{\mathrm{STAut}_{t}}=t + \sum_{1\leq i \leq k} \mathbb{P}_{A_{t_i}} \).
  By definition, \( \mathrm{SubTreeSeries}_t=t+\sum_{1\leq j \leq k} \mathrm{SubTreeSeries}_{t_j} \).
  Furthermore, by induction hypothesis, \( \mathbb{P}_{\mathrm{STAut}_{t_i}}= \mathrm{SubTreeSeries}_{t_i} \).
  Therefore, it holds that
  \begin{equation*}
    \mathbb{P}_{\mathrm{STAut}_{L}}=t+\sum_{1\leq j \leq k} \mathrm{SubTreeSeries}_{t_j}=\mathrm{SubTreeSeries}_L.
  \end{equation*}
\end{proof}

\section{Subtree Kernel Computation}\label{sec algo}
In this section, we present algorithms that allow us to compute efficiently tree kernels using the Hadamard product of tree automata.
\subsection{SubTree Automata Construction}

An automaton \( A \) is said to be an \emph{ST automaton} if it is isomorphic to some \( \mathrm{STAut} \).

In this section, we present an incremental algorithm that constructs an ST automaton from a finite set of trees.

By construction, an ST automaton  is homogeneous, i.e., all transitions entering a state \( q \) have the same label. So we can define a function
\( h \) that associates with a state \( q \) its symbol \( h(q) \). For example in Figure~\ref{fig ex RWTA} we have \( h(1)=a \),
\( h(2)=f \) and \( h(4)=f \).
As \( |\delta(f,q_1,\ldots,q_n)|\leq 1, (q,f,q_1,\ldots,q_n)\in\delta \Leftrightarrow \delta(f,q_1,\ldots,q_n)=\{q\} \). So, we define
\( \delta^{-1}(q) \) as \( f(q_1,\ldots,q_n) \).

\begin{example}\label{ex3}
  The transitions of the RWTA recognizing the tree \( f(a,g(a)) \) of Example~\ref{exp tree} are
  \( \{(1,a),(2,g,1),(3,f,1,2)\} \).
  The function \( \delta^{-1} \) is defined by
  \begin{align*}
    \delta^{-1}(1) & =a ,     &
    \delta^{-1}(2) & =g(1),   &
    \delta^{-1}(3) & =f(1,2).
  \end{align*}
\end{example}

The transitions \( \delta \) can be represented by a bideterministic automaton (see Figure~\ref{pref_treef}). Thus, the computation of the image of \( f(q_1,\ldots,q_n) \) by \( \delta \) (i.e. \( \delta(f,q_1,\ldots,q_n) \)) can be done in \( O(n) \) where \( n \) is the rank of \( f \). Furthermore, the function \( \delta^{-1}(q) \) can be computed using this bideterministic automaton in the same time complexity.


Let us consider the two ST automata \(A_X\) and \(A_Y\) associated respectively with the sets \( X \) and \( Y \) defined by
\begin{align*}
  A_{X} & =(\Sigma,Q_X,Q_{m_X},\nu_X,\delta_X,\delta^\bot_X), \\
  A_{Y} & =(\Sigma,Q_Y,Q_{m_Y},\nu_Y,\delta_Y,\delta^\bot_Y).
\end{align*}
To compute the sum of the ST automata \( A_X \) and \( A_Y \),
we define the
partial
function \( \phi \) from \( Q_X \) to \( Q_Y \) which identifies states in \( A_X \) and \( A_Y \) that have the same down language, i.e., for \( p\in Q_X \) and \( p'\in Q_Y \), \( \phi(p)=p' \) \( \Leftrightarrow \) \( L_p=L_{p'} \). Notice that \( \phi \) is a well-defined function, indeed for  an ST automaton, one has for all distinct states \( p \) and \( q \), \( L_p \neq  L_q \).

Algorithm~\ref{algo-union} loops through the
transitions of \( A_X \)
and computes at each step the function \( \phi \)
if possible.
So if the current element \( \alpha \) is \( f(q_1,\ldots, q_n) \) then \( \phi(q_i) \), \( 1\leq i\leq n \) must be defined.
In order to ensure this property,
transitions of \( A_X \)
can be
stored in an ordered list \( \mathrm{OL}_\delta \) as follows:
If \( (q,f,q_1,\ldots,q_n)\in \delta \) then \( \delta^{-1}(q_i)<\delta^{-1}(q) \) for all \( 1\leq i\leq n \).
For example for the transition function of the Example~\ref{ex3}, the ordered list \( \mathrm{OL}_\delta \) is  \( [a,g(1),f(1,2)] \).



\begin{figure}[H]
  \begin{center}
  \scalebox{0.8}{
    \begin{tikzpicture}
      \node [state](A)                    {\( q_a \)};
      \node[state,initial above]         (B) [above right of=A] {\( I \)};
      \node[state]         (C) [below of=B] {\( q_g \)};
      \node[state]         (D) [right of=C] {\( q_f \)};
      \node[state]         (E) [below of=C] {\( q^1_a \)};
      \node[state]         (F) [below of=E] {\( q_{\delta} \)};
      \node[state]         (G) [below of=D] {\( q^1_g \)};
      \node[state]         (H) [below left of=G,node distance=1cm] {\( q^2_a \)};
      \path (B) edge              node {a} (A)
      edge              node {g} (C)
      edge              node {f} (D)
      (C) edge              node {\( 1 \)} (E)
      (A) edge[bend right] node{\( \underline{\mathbf{1}} \)} (F)
      (H) edge              node{\( \underline{\mathbf{3}} \)} (F)
      (E) edge              node {\( \underline{\mathbf{2}} \)} (F)
      (D) edge             node {\( 1 \)} (G)
      (G) edge              node {\( 2 \)} (H)
      ;
    \end{tikzpicture}
}
    \caption{Bideterministic automaton associated to \( \delta \).}%
    \label{pref_treef}
  \end{center}
\end{figure}
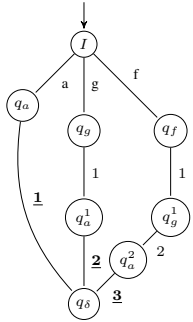
%
%



\begin{algorithm}[ht]
\small{
  \KwIn{ ST Automata \( A_{X} \) and \( A_{Y} \)}
  \KwOut{ST Automaton \( A_{X}+A_{Y} \)}
  \For{\( \alpha\in \Sigma_0 \)}{
    \If{\( \delta_X(\alpha)\text{ and } \delta_Y(\alpha) \) exist}
    {\( \phi(\delta_X(\alpha))\leftarrow \delta_Y(\alpha) \) \;
    }
  }
  \( it \leftarrow \mathrm{OL}_{\delta_X}.\mathrm{getIterator}() \) \;
  \While{\( \mathrm{it.hasNext}() \)}
  {
  \( \alpha \leftarrow \mathrm{it.Next}() \); /*\( \alpha=f(q_1,\ldots,q_n)* \)/

  \If{\( \delta_Y(f(\phi(q_1),\ldots,\phi(q_n))) \) exists}
  {\( \phi(\delta_X(\alpha))\leftarrow \delta_Y(f(\phi(q_1),\ldots,\phi(q_n))) \) \;
    \( \nu_Y(\phi(\delta_X(\alpha)))\leftarrow\nu_Y(\delta_Y(\alpha))+ \nu_X(\delta_X(\alpha)) \) \;
  }
  \Else{
  Add  a new state \( q' \) to \( Q_Y \) \;
  \If{\( \delta_Y(f(\phi(q_1),\ldots,\phi(q_n)))\in Q_{m_X} \)} {{Add \( q' \) to \( Q_{m_Y} \) \;}}
  \( \mathrm{OL}_{\delta_Y} . \mathrm{Add}(f(\phi(q_1),\ldots,\phi(q_n)),q') \) \;
  \( \delta_Y . \mathrm{Add}(f(\phi(q_1),\ldots,\phi(q_n)),q') \) \;
  \( \nu_Y(q')\leftarrow\nu_X(\delta_{X}(\alpha)) \) \;
  \( \phi(\delta_{X}(\alpha))\leftarrow q' \) \;
  }

  }
  }
  \caption{Union of \( A_X \) and \( A_Y \)}%
  \label{algo-union}
\end{algorithm}
Algorithm~\ref{algo-union} is performed in \( O(|\delta_X|) \) time complexity.

\begin{proposition}
  Let \( X \) and \( Y \) be two finite sets of trees and let \( A_X \) and \( A_Y \) the corresponding ST automata. Then, the ST  automaton \( A_X+A_Y \) can be computed in \( O(\min(|A_X|,|A_Y|)) \) time complexity.
\end{proposition}

Notice that an ST automaton associated with a tree \( t=f(t_1,\ldots, t_n) \) can be computed using Algorithm~\ref{algo-union}. Indeed,
first we construct the automaton associated with the set \( \{t_1,\ldots, t_n\} \). Next if \( t_i \) is recognized at the state
\( q_{k_i} \) for \( 1\leq i\leq n \), we add, a new state \( q \), a new  transition \( (q,f,q_{k_1},\ldots, q_{k_n} ) \) and set \( Q_m=\{q\} \) and \( \nu(q)=1 \).

\begin{theorem}\label{cor tps cons al}
  Let \( \Sigma \) be an alphabet.
  Let \( L \) be a finite tree language over \( \Sigma \).
  Then, the ST automaton associated with \( L \) can be computed in \( O(|L|) \) time.

\end{theorem}

\subsection{Hadamard Product Computation}

\begin{definition}
  Let \( X \) and \( Y \) be two finite tree languages. We define:
  \( \mathrm{SubTreeKernel}(X,Y)=\displaystyle\sum_{t\in T_{\Sigma}}\mathrm{SubTreeSeries}_X(t)\odot \mathrm{SubTreeSeries}_Y(t) \), where \( \odot \) is the Hadamard product.
\end{definition}

\begin{example}\label{ex subset tree}
  Let \( \Sigma \) be the  alphabet defined by
  \begin{align*}
    \Sigma_0 & = \{a,b\}, &
    \Sigma_1 & = \{h\},   &
    \Sigma_2 & = \{f\}.
  \end{align*}
  Let us consider the trees $t_1  = f(h(a),f(h(a),b))$,\\ $t_2 = f(h(a),h(b))$ and $t_3  = f(f(b,h(b)),f(h(a),h(b)))$. \\
  Then it can be shown that:
  \begin{gather*}
    \begin{aligned}
      \mathrm{SubTreeSeries}_{t_1} & = t_1+f(h(a),b)+2h(a)+2a+b \\
      \mathrm{SubTreeSeries}_{t_2} & = t_2+h(b)+h(a)+a+b        \\
      \mathrm{SubTreeSeries}_{t_3} & = t_3+f(b,h(b))+t_2        \\
                                  & \quad +2h(b)+h(a)+3b+a
    \end{aligned}\\
    \begin{aligned}
      \mathrm{SubTreeSeries}_{\{t_1,t_2\}} & = t_1+t_2+f(h(a),b)     \\
                                           & \quad +3h(a)+h(b)+3a+2b  \\
  \mathrm{SubTreeSeries}_{t_3} & = t_3+f(b,h(b))+t_2        \\
                                  & \quad +2h(b)+h(a)+3b+a 
    \end{aligned}\\
    \begin{aligned}
      \mathrm{SubTreeSeries}_{\{t_1,t_2\}} & \odot \mathrm{SubTreeSeries}_{t_3} \\
                                           & = t_2+2h(b)+3h(a)+6b+3a
    \end{aligned}\\
    \begin{aligned}
      \mathrm{SubTreeKernel}(\{t_1,t_2\},\{t_3\}) & = 15
    \end{aligned}
  \end{gather*}
\end{example}

The following proposition shows how \( \mathrm{SubTreeKernel}(X,Y) \)
can be computed from the
automata \( A_X \) and \( A_Y \).

\begin{proposition}
  Let \( A_X=(\Sigma,Q_X,Q_{m_X},\nu_X,\delta_X\uplus\delta^\bot_X) \) and \( A_Y=(\Sigma,Q_Y,Q_{m_Y},\nu_Y,\delta_Y\uplus\delta^{\bot}_Y) \) be two RWTAs.
  The  RWTA \( A_X\odot A_Y=(\Sigma,Q_{X\odot Y}\cup \{\bot \},Q_{m_{X\odot Y}},\nu_{X\odot Y},\delta_{X\odot Y}\uplus \delta^{\bot}_{X\odot Y}) \)
  where:
  \begin{itemize}
    \item \( Q_{X\odot Y}=Q_X\times Q_Y \)
    \item 
          \( \delta_{X\odot Y}= \hspace{-1cm} \bigcup\limits_{\substack{
              f\in \Sigma_k,\\
              (p,f,p_1,\ldots,p_k)\in \delta_X,\\
              (q,f,(q_1,\ldots,q_k))\in\delta_Y
            }} \hspace{-1cm} \{((p,q),f,(p_1,q_1),\ldots,(p_k,q_k))\}  \)
    \item
          \( \delta^{\bot}_{X\odot Y}=
          \{(\bot,\bot)\} \cup \hspace{-1cm}
          \bigcup\limits_{\substack{
              f\in \Sigma_k,\\
              (p,f,p_1,\ldots,p_k)\in \delta_X,\\
              (q,f,(q_1,\ldots,q_k))\in\delta_Y
            }}
          \hspace{-1cm}
          \{((p,q),f,\bot,\ldots,\bot)\}
          \)
    \item \( \forall (q_1,q_2)\in Q_{X\odot Y} \), \( \nu_{X\odot Y}((q_1,q_2))=\nu_X(q_1)\times\nu_Y(q_2) \).
    \item \( (p,q)\in Q_{m_{X\odot Y}} \mbox{ if } p\in Q_{m_X}\land q\in Q_{m_Y} \),
  \end{itemize}
  realizes the tree series \( \mathbb{P}_{A_X}\odot \mathbb{P}_{A_Y} \).
\end{proposition}

\begin{corollary} Let \( A_X \) and \( A_Y \) be two
  \( \alpha \) automata. Then,
  \begin{equation*}
    \mathrm{SubTreeKernel}(X, Y)=\displaystyle\sum_{t\in T_{\Sigma}}\mathbb{P}_{A_X\odot A_Y}(t).
  \end{equation*}
\end{corollary}
For an efficient computation,
we must compute just
the accessible part of the automaton  \( A_X\odot A_Y \).
The size of this accessible part is equal to 
\( |\mathrm{SubTreeSet}(X)\cap \mathrm{SubTreeSet}(Y)| \).
Notice that \( |\mathrm{SubTreeSet}(X)\cap \mathrm{SubTreeSet}(Y)| \leq |A_X|+|A_Y| \).

The following algorithm (Algorithm~\ref{algo-Intersection2}) computes the automaton \( A_X \odot A_Y \)
for two ST automata.
\begin{algorithm}[ht]
\small{
  \KwIn{ ST Automata \( A_{X} \) and \( A_{Y} \)}
  \KwOut{\( A_X\odot A_Y \)}
  \For{\( \alpha\in \Sigma_0 \)}{
    \If{\( (\delta_X(\alpha)\text{ and } \delta_Y(\alpha)) \) exist}
    {\( \phi(\delta_X(\alpha))\leftarrow \delta_Y(\alpha) \) \;
    }
  }
  \( it\leftarrow \mathrm{OL}_{\delta_X}. \mathrm{getIterator}() \) \;
  \While{\( it.\mathrm{hasNext}() \)}
  {
  \( \alpha \leftarrow it.\mathrm{Next}() \) \;
  \If{\( f(\phi(q_1),\ldots,\phi(q_n)) \) exist}{
    Let \( \alpha=f(q_1,\ldots,q_n) \) \;
    Let \( p=\delta_{Y}(f(\phi(q_1),\ldots,\phi(q_n))) \) \;
    \( \phi(\delta_{X}(\alpha))\leftarrow p \) \;
    \( \nu_X(\delta_{X}(\alpha))=\nu_X(\delta_{X}(\alpha))\times\nu_Y(p) \) \;
    }
    \Else{
      \( \delta_X.\mathrm{Remove}(\alpha,\delta_X(\alpha)) \) \;
      \( it.\mathrm{Remove}(\alpha) \) \;
      \( \phi(\delta_{X}(\alpha))=\bot \) \;
    }
  }
  }
  \caption{The Subtree Automaton \( A_X\odot A_Y \).}%
  \label{algo-Intersection2}
\end{algorithm}

\begin{proposition}
  Let \( A_X \) and \( A_Y \) be two ST automata associated respectively with the sets of trees \( X \) and \( Y \).
  The automaton \( A_{X}\odot A_{Y} \) can be computed in time\\
   $ O(\min(|\mathrm{SubTreeSet}(X)|, |\mathrm{SubTreeSet}(Y)|))$.
\end{proposition}

\subsection{Kernel Computation}

\begin{proposition}
  Let \( X \) and \( Y \) be two finite tree languages. Let \( Z \) be the accessible part of the ST automaton \( A_X\odot A_Y \). Then,

  $ \mathrm{SubTreeKernel}(X, Y)=\sum_{q\in Z} \nu(q)$.
\end{proposition}
As the size of the accessible part of \( A_X\odot A_Y \) is bounded by  \( O(\min(|A_X|,|A_Y|)) \), we can state the following proposition.

\begin{proposition}
  Let \( X \) and \( Y \) be two finite tree languages, and  \( A_X,  A_Y \) be their associated ST automata.
  Then,  the subtree kernel \( \mathrm{SubTreeKernel}(X,Y) \) can be computed in time\\
   $ O(\min(|A_X|, |A_Y|))$.
\end{proposition}
Finally, we get our main result.
\begin{theorem}  Let \( X \) and \( Y \) be two finite tree languages. The subtree kernel \( \mathrm{SubTreeKernel}(X,Y) \) can be computed in \( O(|X|+|Y|) \) time and space complexity.
\end{theorem}
This is due to the fact that the incremental construction of \( A_X \) (resp. \( A_Y) \) from the set \( X \) (resp. \( Y \)) needs \( O(|X|) \) (resp. \( O(|Y|) \)) time and space complexity with  \( A_X\leq |X| \) and \( A_Y\leq |Y| \).

\section{Experiments and results}\label{sec experiments}
This section includes extensive and comparative experiments to evaluate the efficiency  of the ST kernel computation based on RWTA
in terms of the reduction of its representation and the time of execution.
From  an algorithmic point of view,
the available real-world data sets for this task are standard benchmarks for learning on relatively small trees.
They do not cover a wide variety of tree characteristic that is  necessary for a deep algorithm analysis purpose.

To verify our method efficiency,
experiments are conducted on synthetic  unordered tree data sets randomly generated
as in~\cite{AiolliMSM06}\footnote{\small http://www.math.unipd.it/\( \sim \)dasan/pythontreekernels.htm} with various combinations of attributes
including the  alphabet size \( |{\cal F}| \), the maximum of the alphabet arity \( A \), and the maximum tree depth \( D \). For each combination of
a data set parameters \( (|{\cal F}|,A,D) \),
we generate uniformly and randomly a  tree set with cardinal \( 100 \) and   average  size \( N \).

The following table summarizes the dataset parameters used in our experiments where the dataset DS1 (respectively DS2 and DS3) is composed
of five (respectively four and seven) tree sets,
each having a cardinal
equal to 100, obtained by
varying the maximum tree depth \( D \) (respectively the maximum of the alphabet arity \( A \) and both parameters \( (A,D) \)).
\begin{table}[ht]
\caption{Datasets used in our experiments.}
  \centering
  \begin{tabular}{cccccc}
    \toprule
        & Size      & \( |{\cal F}| \) & \( A \)      & \( D \)       & \( N \)        \\
    \midrule
    DS1 & \( 500 \) & \( 2 \)          & \( 5 \)      & \( [5,100] \) & \( [8,1526] \) \\
    \midrule
    DS2 & \( 400 \) & \( 2 \)          & \( [5,20] \) & \( 5 \)       & \( [8,1196] \) \\
    \midrule
    DS3 & \( 700 \) & \( 2 \)          & \( [2,15] \) & \( [5,100] \) & \( [5,478] \)  \\
    \bottomrule
  \end{tabular}

\end{table}

All the algorithms are implemented in C++11 and Bison++ parser.
Compilation and
assembly were made in gnu-gcc. All experiments were performed on a laptop with Intel Core i5--4770K (3.5GHz) CPU and 8Gb RAM\@.
Source code 
can be found \href{https://github.com/ouardifaissal/Subtree-kernel-computation-using-RWTA}{here~\cite{Ouardi2021}}.

For each combination of dataset parameters \( (|{\cal F}|,A, D) \), we evaluated the ST kernel of all 4950 possible tree pairs. Then,
we derive the average computation time and the average number of states of the constructed RWTA on all the tree pairs.

The obtained results of the conducted experiments, from Figures~\ref{fig1},~\ref{fig2} and~\ref{fig3}, show clearly that our approach
is linear, asymptoticly  logarithmic,  w.r.t. the sum of trees size  and more efficient than the existing methods for a
wide
variety of trees. These results can be explained by the fact that our approach is output-sensitive.
In addition, it produces a
compact representation of the ST kernel  that can be used  efficiently in  incremental learning algorithms.
\vspace{-0.5cm}
\section{Conclusion and Perspectives}\label{sec conclusion}
\vspace{-0.2cm}
In this paper, we defined new weighted tree automata.
Once these definitions stated, we made use of these new structures in order to compute the subtree kernel of two finite tree languages efficiently.\\
Our approach can be applied to compute other  distance-based tree kernel like the SST kernel, the subpath tree kernel, the topological tree kernel and  the Gappy tree kernel. 
The next step of our work is to apply our constructions in order to efficiently compute these kernels using a unified framework based on weighted tree automata. However, this application is not so direct since it seems that the SST series may not be sequentializable w.r.t. a linear space complexity.
Hence we have to find different techniques, like extension of lookahead determinism~\cite{HW08} for example.
\begin{figure}[H]
  \begin{tabular}{lc}
  (a) &  \includegraphics[width=6.8cm]{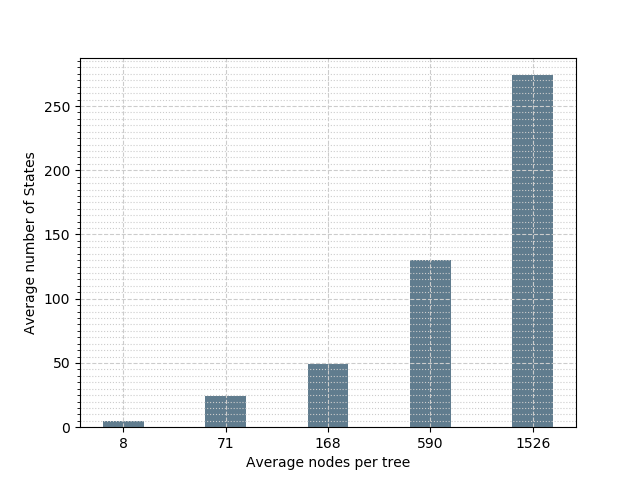} \\
  (b) & \includegraphics[width=6.8cm]{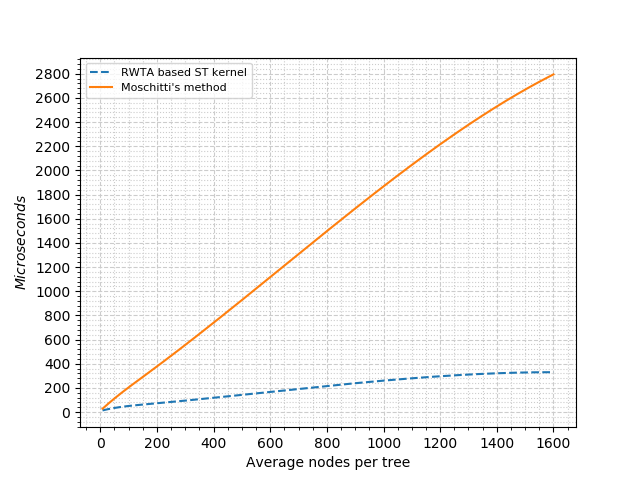} 
  \end{tabular}
  \caption{(a) The reduction ratio of a tree representation using RWTA approach for data set DS1.
    (b) The average  computation time of the ST kernel when varying the tree depth for data set DS1.}%
  \label{fig1}
\end{figure}

\vspace{-1cm}
\begin{figure}[H]
  \begin{tabular}{lc}
  (a) & \includegraphics[width=6.8cm]{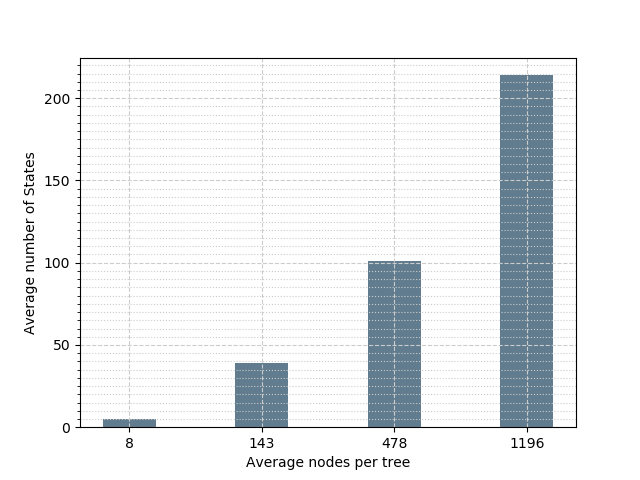} \\   (b) & \includegraphics[width=6.8cm]{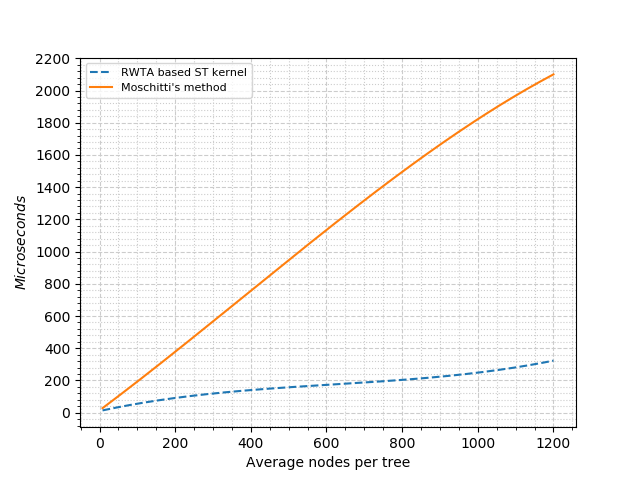}
  \end{tabular}
  \caption{(a) The reduction ratio of a tree representation using RWTA approach for data set DS2.
    (b) The average  computation time of the ST kernel  when varying the alphabet arity for data set DS2.}%
  \label{fig2}
\end{figure}

\begin{figure}[H]
  \begin{tabular}{lc}
    (a) & \includegraphics[width=7cm]{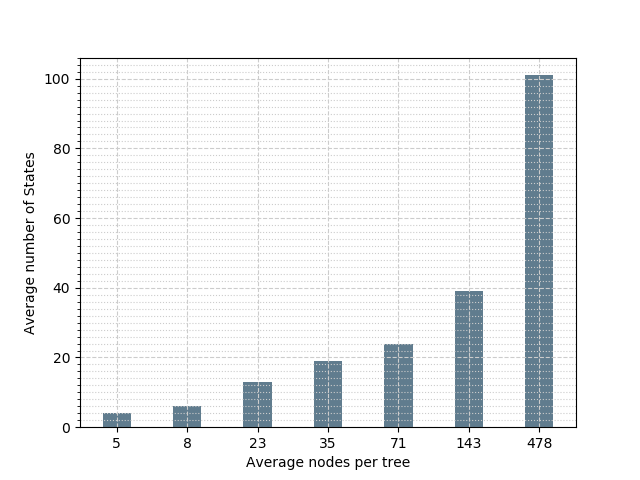}  \\ (b) & \includegraphics[width=7cm]{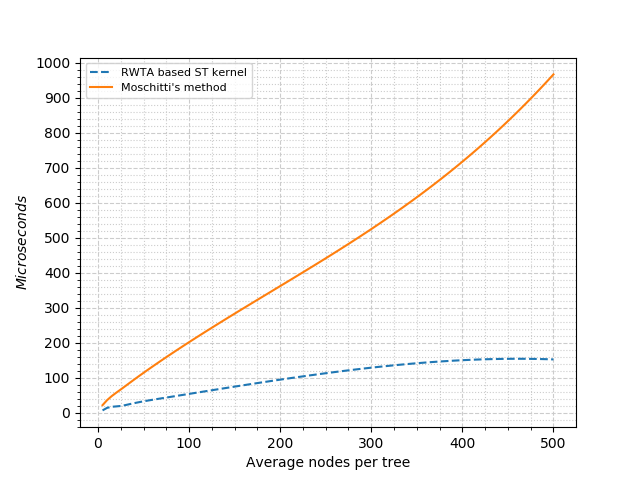} 
  \end{tabular}
  \caption{(a) The reduction ratio of a tree representation using RWTA approach for data set DS3.
    (b) The average  computation time of the ST kernel  when varying the tree depth and the alphabet arity for data set DS3.}%
  \label{fig3}
\end{figure}
\bibliographystyle{elsarticle-num}
\bibliography{RWTA_STKERNEL}
\end{document}